\documentclass[arxiv]{jmlr}

\usepackage[hpos=300px,vpos=70px]{draftwatermark}
\SetWatermarkText{\test}
\SetWatermarkScale{1}
\SetWatermarkAngle{0}
\usepackage{mathtools}

\usepackage{longtable}% for long tables

\usepackage{booktabs}
\usepackage[load-configurations=version-1]{siunitx} % newer version
\theorembodyfont{\upshape}
\theoremheaderfont{\scshape}
\theorempostheader{:}
\theoremsep{\newline}

\jmlrvolume{}
\firstpageno{1}

\jmlryear{2024}
\jmlrworkshop{Symmetry and Geometry in Neural Representations - NeurIPS Workshop - Extened Abstract Track}

\title{Convergence of Manifold Filter-Combine Networks}

   \author{\Name{David R. Johnson} \Email{davejohnson408@u.boisestate.edu}\\
   \Name{Joyce Chew} \Email{joycechew@math.ucla.edu}\\
   \Name{Siddharth Viswanath} \Email{siddharth.viswanath@yale.edu }\\
   \Name{Edward De Brouwer} \Email{edward.debrouwer@gmail.com}\\
   \Name{Deanna Needell} \Email{deanna@math.ucla.edu }\\
   \Name{Smita Krishnawamy} \Email{smita.krishnaswamy@yale.edu}\\
   \Name{Michael Perlmutter} \Email{mperlmutter@boisestate.edu}\\
 }

\begin{document}

\maketitle

\begin{abstract}
In order to better understand manifold neural networks (MNNs), we introduce Manifold Filter-Combine Networks (MFCNs). The filter-combine framework parallels the popular aggregate-combine paradigm for graph neural networks (GNNs) and naturally suggests many interesting families of MNNs which can be interpreted as the manifold analog of various popular GNNs. We then propose a method for implementing MFCNs on high-dimensional point clouds that relies on approximating the manifold by a sparse graph. We prove that our method is consistent in the sense that it converges to a continuum limit as the number of data points tends to infinity.
\end{abstract}

%\acks{Acknowledgements go here.}

\section{Introduction}

Geometric deep learning \citep{bronstein2017geometric,bronstein2021geometric} aims to extend the success of deep learning from data such as images, with a regular grid-like structure, to more irregular domains such as graphs and manifolds. Notably, graph neural networks (GNNs), e.g.,  have rapidly emerged as an extremely active area of research  \citep{ZHOU202057}. 

By contrast, the manifold side of geometric deep learning is much less explored, and 
much of the existing literature on manifold deep learning  
is limited to 2D surfaces  \citep{masci2015geodesic,Masci:geoCNN2015, schonsheck2022parallel}  and cannot be applied to higher-dimensional manifolds. 
 This is despite the fact that unsupervised manifold learning algorithms \citep{coifman:diffusionMaps2006,maaten:tSNE2008} are commonly used for representing higher-dimensional data \citep{van2018recovering,  moyle2021structural}).  

 Inspired by the successes of GNNs and manifold learning, several recent works have proposed  \emph{manifold neural networks} (MNNs) \citep{wang2021stability,wang2021stabilityrel}  
 that define convolution in terms of a manifold Laplacian $\mathcal{L}$,  
 paralleling spectral GNNs \citep{Defferrard2018,Levie:CayleyNets2017} that define convolution in terms of the graph Laplacian $\mathbf{L}$. 
 Moreover, several recent papers \cite{CHEW2024101635,chew2022convergence,chew2022manifold,wangsparseconvergence} have introduced numerical methods for implementing MNNs on point clouds satisfying the manifold hypothesis and establish the accuracy and statistical consistency of these methods by proving they converge to a continuum limit as the number of data points tends to infinity under various assumptions.

In this work, in order to better understand  MNNs, we introduce 
\emph{Manifold Filter-Combine Networks}. The filter-combine paradigm parallels the aggregate-combine framework introduced in \citet{xu2018how} to understand GNNs. It leads one to consider many interesting classes of MNNs which may be thought of as the manifold counterparts of various popular GNNs. We provide sufficient conditions for such networks to converge to a continuum limit as the number of sample points, $n$, tends to infinity.  
Notably, our analysis shows that if the weights of the network are properly normalized, then the rate of convergence depends linearly on the depth of the network, in contrast to previous results exhibiting exponential dependence.

\section{Background}

\subsection{Graph Signal Processing and Spectral Graph Neural Networks}\label{sec: background gsp}
Graph signal processing (GSP) \cite{shuman:emerging2013} extends Fourier analysis to graphs. For a function (signal) $\mathbf{x}$ defined on the vertices of a graph $G=(V, E)$, $V=\{v_1,\ldots,v_n\}$, one can define its graph Fourier transform by 
$
\widehat{\mathbf{x}}(i)=\langle\mathbf{u}_i,\mathbf{x}\rangle_2,
$
where $\mathbf{u}_1,\ldots,\mathbf{u}_n$ is an  orthonormal basis of eigenvectors
for the graph Laplacian $\mathbf{L}=\mathbf{D}-\mathbf{A}$, $\mathbf{L}\mathbf{u}_i=\lambda_i\mathbf{u}_i$.\footnote{Here and throughout, we identify the function $\mathbf{x}$, with the vector $\mathbf{x}\in\mathbb{R}^n$, $x_i=\mathbf{x}(v_i)$.}  
Since the $\mathbf{u}_i$ form an orthonormal basis, we obtain the Fourier inversion formula
$
\mathbf{x} 
=\sum_{i=1}^n \widehat{\mathbf{x}}(i)\mathbf{u}_i.
$
We can then define convolutional operators in the Fourier domain $
w(\mathbf{L})\mathbf{x}= \sum_{i=1}^n w(\lambda_i)\widehat{\mathbf{x}}(i)\mathbf{u}_i
$, and use these convolutions to define spectral GNNs such as ChebNet \citep{Defferrard2018}.

\subsection{Manifold Learning}\label{sec: background manifold learning}

Manifold learning algorithms \citep{lin2015geometric,Moon2018Manifold} aim to detect non-linear structure in the data, analogous to the manner in which PCA is used to detect linear structure.
Given a high-dimensional point cloud
 $\{x_i\}_{i=1}^n\subseteq \mathbb{R}^D$, which is assumed to lie upon an unknown $d$-dimensional  manifold $\mathcal{M}$ ($d\ll D)$,
they aim to produce a low-dimensional representation of the data points $x_i$ that approximates the intrinsic geometry of~$\mathcal{M}$.

Many popular manifold learning algorithms such as Diffusion Maps \citep{coifman:diffusionMaps2006} and Laplacian Eigenmaps \citep{belkin:laplacianEigen2003} operate by constructing a graph $G_n=(V_n, E_n)$ where the vertices are the data points, i.e., $V_n=\{x_i\}_{i=1}^n$.  
For instance, 
Laplacian Eigenmaps  map each $x_j\in\mathbb{R}^D$ to the point $(\phi_2^n(j),\ldots,\phi^n_{m+1}(j))\in\mathbb{R}^m$,
where $\phi^n_1,\ldots,\phi^n_n$ are the eigenvectors of the graph Laplacian $\mathbf{L}_n.$\footnote{The first eigenvector is omitted because it is constant.}  
In order to justify the intuition that Laplacian Eigenmaps and related algorithms capture the intrinsic geometry of the data, one may aim to prove that the graph Laplacian $\mathbf{L}_n$ converges to a manifold Laplacian such as the (negative) Laplace-Beltrami Operator, $-\Delta=-\text{div}\circ\nabla$ as the number of points $n\rightarrow \infty.$  Results along these lines have been established in numerous works such as \citet{dunson2021spectral, cheng2022eigen, Calder2019,belkin2008towards}.

\section{Methods}
\subsection{Signal Processing and Spectral Convolution on Manifolds}

Let $\mathcal{M}$ be a compact, connected, $d$-dimensional Riemannian manifold embedded in $\mathbb{R}^D$, $D\gg d$. 
Let $\mu$ be a probability distribution on $\mathcal{M}$ with a smooth, non-vanishing density $\rho$, and let $L^2(\mathcal{M})$ denote the set of functions where $\|f\|_{L^2(\mathcal{M})}^2=\langle f,f\rangle_{L^2(\mathcal{M})}=\int_{\mathcal{M}}|f|^2d\mu<\infty$.  

We let $\mathcal{L}=-\frac{1}{2\rho}\text{div}(\rho^2\nabla f)$ denote the weighted manifold Laplacian.  It is known that $\mathcal{L}$ has an orthonormal basis of eigenfunctions $\mathcal{L}\phi_i=\lambda_i\phi_i$. Thus, we
  define a generalized Fourier series by
$
\widehat{f}(i)=\langle f,\phi_i \rangle_{L^2(\mathcal{M})},
$
and obtain the Fourier inversion formula 
$
f(x) 
=\sum_{i=1}^\infty \widehat{f}(i)\phi_i(x).
$
We may then define spectral convolution operators\footnote{$w(\mathcal{L})$ is independent of the choice of eigenbasis. See Remark 1 of \cite{CHEW2024101635} for details.}, $w(\mathcal{L})$, for $w\in L^\infty([0,\infty))$ by 
\begin{equation}\label{eqn: specconvcontinuum} w(\mathcal{L})f=\sum_{i=1}^\infty w(\lambda_i) \widehat{f}(i) \phi_i.
\end{equation}

\subsection{Manifold Filter-Combine Networks}\label{sec: MFCN def}

We now introduce the \emph{filter-combine} paradigm, a novel framework, for thinking about manifold neural networks,  
 paralleling the aggregate-combine framework introduced in \citet{xu2018how} in order to understand GNNs.\footnote{We use the term ``filter'' rather than ``aggregate'' because our filters are not required to be localized averaging operations such as those used in common message-passing GNNs.}  

We will assume that our input data is a vector-valued function $F=(f_1,\ldots,f_{C})$, where each $f_i\in L^2(\mathcal{M})$.
Each hidden layer of the network will consist of the following five steps: 
(i) We \textit{filter} each input channel $f_k$ by a family of spectral operators $w_{j,k}(\mathcal{L})$, $1\leq j\leq J$. (ii) For each fixed $j$, we  \textit{combine} the filtered feature functions $\tilde{f}_{j,k}=(w_{j,k}(\mathcal{L})f_k)$ into new feature functions $g_{j,k}$ where each $g_{j,k}$ is a linear combination of the $\tilde{f}_{j,k}$.   (iii) For each fixed $k$, we perform a \textit{cross-filter convolution}  that maps $\{ g_{j,k}\}_{j=1}^J$ to $\{\tilde{g}_{j,k}\}_{j=1}^{J'}$ where each $\tilde{g}_{j,k}$ is a linear combination of the $ g_{j,k}$.  (iv) We apply some \textit{non-linear, nonexpansive pointwise activation function} $\sigma$ to each of the $\tilde{g}_{j,k}$,  to obtain $h_{j,k}=\sigma\circ \tilde{g}_{j,k}$. (v) Lastly, we \textit{reshape} the $\{h_{i,j}\}_{1\leq i \leq \tilde{C},1\leq j\leq J'}$ into $\{f'_i\}_{i=1}^{C'}$, where $C'=\tilde{C}J'$.  We note that each of these steps can be effectively omitted by choosing parameters in a suitable manner (e.g., choosing certain matrices to be the identity). For further details and discussion, please see Appendix \ref{app: MFCN}.

We shall refer to networks constructed using the layers above as Manifold Filter-Combine Networks (MFCNs). 
As illustrated in the examples below, the MFCN framework naturally allows one to consider many different subfamilies of networks. Indeed, for nearly any (spectral) GNN, there is a corresponding MFCN. Moreover, if desired, one could adapt the MFCN framework to consider the counterparts of non-spectral GNNs by allowing the filters in step (i) to be generic linear operators on $L^2(\mathcal{M})$.

\begin{example}[MCNs]\label{ex: mcn}
In order to obtain a network analogous to GCN \cite{kipf2016semi}, we let $\mathcal{A}=w(\mathcal{L})$ for a decreasing function $w$, e.g., $w(\lambda)=e^{-\lambda}$, so that $\mathcal{A}$ may be thought of as a low-pass filter.
We will omit cross-filter convolutions (step (iii)) and use a learnable weight matrix $\Theta$ to combine the filtered features in step (ii). We thus obtain $f_k^{\ell+1}=\sum_{i=1}^{C_\ell}\theta_{i,k}\widetilde{A}f_k$ which may be written compactly as 
$F^{(\ell+1)}=\sigma\left(\widetilde{A}F^{(\ell)}\Theta^{(\ell)}\right)$.
\end{example}

\begin{example}[Manifold ChebNets]\label{ex: chebnet}
In order to obtain a network analogous to ChebNet \citep{Defferrard2018}, one can construct a network where the filters take the form $p_{j,k}(\mathcal{P})$, where each $p_{j,k}$ is a polynomial and $\mathcal{P}=e^{-\mathcal{L}}$ is the heat-kernel. (Note that $\mathcal{P}$ has the same eigenfunctions as $\mathcal{L}$ and that its eigenvalues are given by $0\leq \omega_k=e^{-\lambda_k}\leq 1$. Therefore, polynomials of $\mathcal{P}$ are still spectral filters of the form \eqref{eqn: specconvcontinuum}.)
\end{example}

\begin{example}[The Manifold Scattering Transform]\label{ex: scat}
The manifold scattering transform \citep{perlmutter:geoScatCompactManifold2020} is a hand-crafted, wavelet-based method for deep learning on manifolds inspired by analogous constructions for Euclidean data
\citep{mallat:scattering2012}   
and graphs \citep{gama:stabilityGraphScat2019,gama:diffScatGraphs2018,zou:graphCNNScat2018,gao:graphScat2018}.
Here, we omit steps (ii) and (iii) and consider a family of wavelets $\{w_j(\lambda)\}_{j=1}^J$, chosen to be a band-pass filters such as or $w_j(\lambda)=e^{-2^{j-1}\lambda}-e^{-2^j\lambda}$. One then defines $\ell$-th order \emph{scattering coefficients} by  
$U[j_1,\ldots,j_\ell]f_k=\sigma(w_{j_{\ell}}(\mathcal{L})U[j_1,\ldots,j_{\ell-1}]f_k),$ $U[j_1]f_k=\sigma(w_{j_{\ell}}(\mathcal{L})f_k)$.
We could also consider variations which include steps (ii) and (iii), paralleling analogous graph constructions \citep{tong2022learnable,wenkel2022overcoming}.
\end{example}

\subsection{Implementation on MFCNs on Point Clouds}\label{sec: MFCN discrete implementation}
We now consider the setting where we do not know the manifold $\mathcal{M}$, but are merely given finitely many sample points $\{x_i\}_{i=1}^n\in\mathbb{R}^D$
 (i.e., a point cloud) assumed to lie upon an unknown manifold $\mathcal{M}$ sampled from density $\rho$.  
 We construct a graph $G_n=(V_n, E_n)$, whose vertices are the sample points, and edges are defined by $\{x_i,x_j\}\in E_n$ if $\|x_i-x_j\|_2<\epsilon$, where we set $\epsilon \sim \left ( \log(n)/n\right )^{\frac{1}{d+4}}$ following the lead of \citet{Calder2019}. %We use the eigenvectors and eigenvalues of the graph Laplacian $\mathbf{L}_n$ to approximate the eigenfunctions and eigenvalues of the manifold Laplacian $\mathcal{L}$. 
We then approximate $\mathcal{L}$ by the graph Laplacian $
\mathbf{L}_n=\frac{d+2}{v_d n\epsilon^{d+2}}\left(\mathbf{D}_{n}-\mathbf{A}_{n}\right) 
$ where $\mathbf{D}_n$ and $\mathbf{A}_n$ are the degree and adjacency matrices of $G_n$ and $v_d$ is the volume of the unit ball in $\mathbb{R}^d$.

For an input signal $f$, we let $\mathbf{x}=P_nf$, where $P_nf\in\mathbb{R}^n$ is the vector  $P_nf(i)=f(x_i)/\sqrt{n}$. 
We then approximate $w(\mathcal{L})f$ by $w(\mathbf{L}_n)\mathbf{x}=\sum_{i=1}^\infty w(\lambda^n_i) \widehat{\mathbf{x}}(i) \phi^n_i$, where $\{\phi_i^n\}_{i=1}^n$ is an 
 an orthonormal basis of eigenvectors, $\mathbf{L}_n \phi_i^n = \lambda_i^n \phi_i^n,$  
and implement the MFCN as in Section \ref{sec: MFCN def}, but with this approximation used in step (i). (See Appendix \ref{app: MFCN} for more details.)

Below, in Theorem \ref{thm: bound given filter bound short}, we prove that 
$
\lim_{n\rightarrow\infty}\|\mathbf{x}_k^{\ell}-P_nf_k^{\ell}\|_2=0,
$
where $\mathbf{x}^\ell_k$ is the $k$-th signal in the $\ell$-th layer of the discrete implementation of the network. 
In other words, with sufficiently many data points, the result of our discrete implementation will be nearly the same as if one implemented the entire network on the manifold $\mathcal{M}$, using global knowledge of both $\mathcal{M}$ and the function $F$ and then discretized the corresponding output. We first prove an intermediate result on the convergence of spectral filters. 

\begin{theorem}\label{thm: Filter Error short}  Let $w:[0,\infty)\rightarrow \mathbb{R}$, $\|w\|_{L^\infty([0,\infty))}\leq 1$, and assume that $w$ is normalized Lipschitz, i.e., $|w(\lambda_1)-w(\lambda_2)|\leq |\lambda_1-\lambda_2|$. Assume that $f$ has $\kappa<\infty$ nonzero Fourier coefficients. 
Then, with probability $1-\mathcal{O}(n^{-9}),$ we have  
\begin{align*}\label{eqn: filter stability with x}
\|w(\mathbf{L}_n)P_nf-\hspace{-.03in}P_nw(\mathcal{L})f\|_2\leq& 
\mathcal{O}\left(\hspace{-.05in}\left( \frac{\log(n)}{n}\right)^{\frac{1}{d+4}}\right)\hspace{-.05in}\|f\|_{L^2(\mathcal{M})}+\mathcal{O}\left(\left(\frac{\log(n)}{n}\right)^{\frac{1}{4}}\right)\|f\|_{L^4(\mathcal{M})},
\end{align*}
where the constants implied  by the big-$\mathcal{O}$ notation depend on $\kappa$ and the geometry of $\mathcal{M}$.
\end{theorem}
An numerical illustration of Theorem \ref{thm: Filter Error short} is provided in Figures \ref{fig:eps_spectral_filter_converg_plot} and \ref{fig:knn_spectral_filter_converg_plot} in Appendix \ref{app: numerics}.\footnote{Code available at \url{https://github.com/dj408/mfcn}} We plot the discretization error $\|w(\mathbf{L}_n)P_nf-\hspace{-.03in}P_nw(\mathcal{L})f\|_2$ for the spectral filter applied to a simple function $f$ constructed as the sum of two spherical harmonics (i.e., eigenfunctions).\footnote{While our method is designed for generic manifolds, here we consider the sphere since there are known formulas for the eigenfunctions which allows us to compare against ground truth.} For further details see Appendix \ref{app: numerics}.  
Next, we use Theorem \ref{thm: Filter Error short} to prove the following result showing $
\lim_{n\rightarrow\infty}\|\mathbf{x}_k^{\ell}-P_nf_k^{\ell}\|_2=0
$.

\begin{theorem}\label{thm: bound given filter bound short}
 Let $w$ and each of the $f_k$ satisfy the assumptions of  Theorem \ref{thm: bound given filter bound short}. Then, 
\begin{align*}\|\mathbf{x}_k^{\ell}-P_nf_k^{\ell}\|_2\leq \ell \left(\mathcal{O}\left(\left( \frac{\log(n)}{n}\right)^{\frac{1}{d+4}}\right)\max_k\|f_k\|_{L^2(\mathcal{M})}+\mathcal{O}\left(\left(\frac{\log(n)}{n}\right)^{1/4}\right)\max_k\|f_k\|_{L^4(\mathcal{M})}\right)\end{align*}
with probability $1-\mathcal{O}(n^{-9})$, where the constants implied by the big-$\mathcal{O}$ notation depend on the geometry of $\mathcal{M}$ and the weights used in steps (ii) and (iii) of the MFCN.
\end{theorem}

\bibliography{main}
\newpage 
\appendix
\section{Further Details on MFCNs}
\label{app: MFCN}

Here, we provide explicit layerwise update rules for MFCNs and also add some further discussion.
In the $\ell$-th layer, given input
$F^{(\ell)}=(f_1^{(\ell)},\ldots,f_{C_\ell}^{(\ell)})$, we define $F^{(\ell+1)}=(f_1^{(\ell+1)},\ldots,f_{C_{\ell+1}}^{(\ell+1)})$ by:
\begin{align*}
\text{Filtering:}&\quad
\tilde{f}^{(\ell)}_{j,k}=w^{(\ell)}_{j,k}(\mathcal{L})f^{(\ell)}_k,\quad 1\leq j \leq J_\ell, 1\leq k\leq C_\ell\\
\text{Combine:}&\quad
g_{j,k}^{(\ell)}=\sum_{i=1}^{C_{\ell}}\tilde{f}^{(\ell)}_{j,i}\theta^{(\ell,j)}_{i,k},\quad 1\leq j\leq J_\ell, 1\leq k \leq C'_\ell\\
\text{Cross-Filter Convolution:}&\quad
\tilde{g}^{(\ell)}_{j,k}= \sum_{i=1}^{J_\ell} \alpha^{(\ell,k)}_{j,i}g_{i,k}, 1\leq j\leq J_\ell',1\leq k\leq C'_\ell\\
\text{Activation:}&\quad
h_{j,k}^{(\ell)}=\sigma^{(\ell)}\circ \tilde{g}_{j,k}^{(\ell)},\quad 1\leq j\leq J_\ell', 1\leq k \leq C'_\ell\\
\text{Reshaping:}&\quad
f^{(\ell+1)}_{(j-1)C_{\ell}+k} = h^{(\ell)}_{j,k}, 1\leq j\leq J_\ell',1\leq k\leq C'_\ell,
\end{align*}
where $C_{\ell+1}=J'_\ell C_{\ell}'$ (and we set $F_0=F,C_0=C$). A graphical depiction of these operations performed is given in Figure \ref{fig:architecture}. We note that the reshaping operator is merely included so that each layer both inputs and outputs a vector of functions, allowing for multiple layers to be stacked upon each other.

\begin{figure}[htbp]
    \centering
    \includegraphics[width=\linewidth]{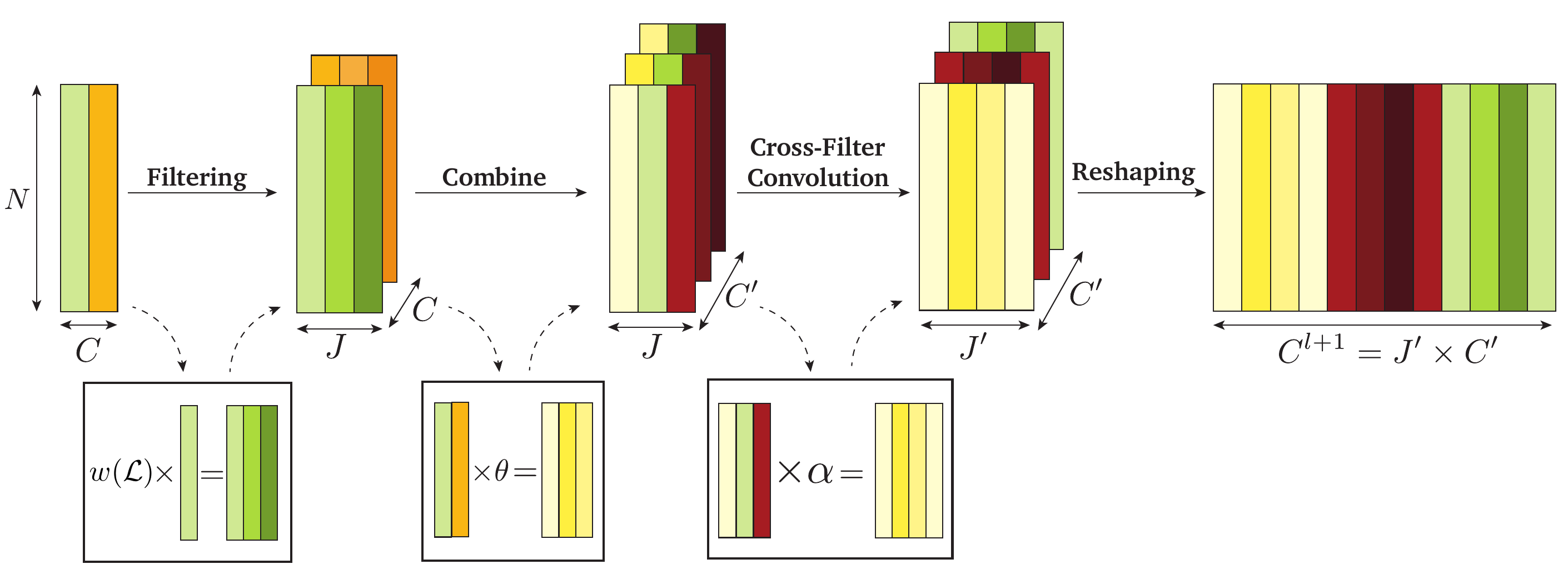}
    \caption{\textbf{Manifold Filter-Combine Network} architecture. Starting from a $C$-dimensional row vector-valued function, each layer in turn filters, combines, convolves channel-wise, applies a point-wise nonlinearity, and reshapes. (For conciseness, we do not visualize the activation step). }
    \label{fig:architecture}
\end{figure}

We also note one may effectively omit the combine step by setting each matrix $\Theta^{(\ell,j)}\coloneqq(\theta_{i,k}^{(\ell,j)})_{1\leq i,k\leq C_\ell}$ equal to the identity. 
Similarly, one may omit the cross-filter convolutions by setting the matrices $(\alpha_{j, i}^{(\ell,k)})_{1\leq i,j\leq J_\ell}$ to the identity.  
Additionally, we observe that because of the generality of our framework, it is possible to write the same network as an MFCN in more than one way.   
For instance, if one omits the cross channel convolutions, uses a shared filter bank $\{w^{(\ell)}(\mathcal{L})_j\}_{1\leq j \leq J}$ (independent of $k$) and chooses the combine step to be independent of $j$ (i.e., $\theta_{i,k}^{(\ell,j)}=\theta_{i,k}^{(\ell)}$) then we have
$
f^{(\ell+1)}_{(j-1)C_\ell+k} = \sigma^{(\ell)}\left(\sum_{i=1}^{C_\ell}w^{(\ell)}(\mathcal{L})_j\theta^{(\ell)}_{i,k}f_i\right),
$
which may also be obtained by using filters of the form 
$\widetilde{w}^{(\ell)}_{(j-1)C_\ell+k,i}(\mathcal{L})=w_{j}(\mathcal{L})\theta^{(\ell)}_{i,k}$ and using a combine step with $\tilde{\theta}_{i,k}^{(\ell,j)}=1$.

In our discrete implementation, we assume we have an  $n\times C$ data matrix $\mathbf{X}=(\mathbf{x}_1,\ldots,\mathbf{x}_C)$, where $\mathbf{x}_k=P_nf_k$, where as before, $P_nf\in\mathbb{R}^n$ is the vector defined by $P_nf(i)=\frac{1}{\sqrt{n}}f(x_i)$. We may then apply the following discrete update rules paralleling those theoretically conducted in the continuum.
\begin{align*}
\text{Filtering:}&\quad
\tilde{\mathbf{x}}^{(\ell)}_{j,k}=w^{(\ell)}_{j,k}(\mathbf{L}_n)\mathbf{x}^{(\ell)}_k,\quad 1\leq j \leq J_\ell, 1\leq k\leq C_\ell\\
\text{Combine:}&\quad
\mathbf{y}_{j,k}^{(\ell)}=\sum_{i=1}^{C_{\ell}}\tilde{\mathbf{x}}^{(\ell)}_{j,i}\theta^{(\ell,j)}_{i,k},\quad 1\leq j\leq J_\ell, 1\leq k \leq C'_\ell\\
\text{Cross-Filter Convolution:}&\quad
\tilde{\mathbf{y}}^{(\ell)}_{j,k}= \sum_{i=1}^{J_\ell} \alpha^{(\ell,k)}_{j,i}\mathbf{y}_{i,k}, 1\leq j\leq J_\ell',1\leq k\leq C'_\ell\\
\text{Activation:}&\quad
\mathbf{z}_{j,k}^{(\ell)}=\sigma\circ \tilde{\mathbf{y}}_{j,k}^{(\ell)},\quad 1\leq j\leq J_\ell, 1\leq k \leq C'_\ell\\
\text{Reshaping:}&\quad
\mathbf{x}^{(\ell+1)}_{(j-1)C_{\ell}+k} = \mathbf{z}^{(\ell)}_{j,k}, 1\leq j\leq J_\ell',1\leq k\leq C'_\ell. 
\end{align*}

We note that the filtering step, in principle, requires the eigendecomposition of the $\mathcal{L}_n$ which bears an $\mathcal{O}(n^3)$ computational cost (if all $n$ eigenvalues are utilized). However, this cost can be reduced by approximating $w_{j,k}(\lambda)$ by a polynomial which will allow one to implement the filters without directly computing an eigendecomposition.

\section{Auxilliary Results}

In order to prove Theorem \ref{thm: Filter Error short}, we must account for three sources of discretization error:
(i) The graph eigenvalue $\lambda_i^n$ does not exactly equal the manifold eigenvalue $\lambda_i$. 
    (ii) The graph eigenvector $\phi_i^n$ does not exactly equal $P_n\phi_i$, the discretization of the true continuum eigenfunction. 
     (iii) The discrete Fourier coefficients $\widehat{\mathbf{x}}(i)$ are not exactly equal to the continuum Fourier coefficients $\widehat{f}(i)$. 

In order to control the first two sources of error, we recall the following result from \cite{Calder2019}.
\begin{theorem}[Theorems 2.4,2.7 of \cite{Calder2019}]\label{thm: recall Calder results}
Assume that $G_n$ is constructed as in Section \ref{sec: MFCN discrete implementation}.
Fix $\kappa>0$ and let $\alpha_{n}=\max_{1\leq i\leq \kappa}|\lambda_i-\lambda_i^n|$ and $\beta_n=\max_{1\leq i\leq \kappa}\|\phi_i^n-P_n\phi_i\|_2$.  Then, with probability $1-\mathcal{O}(n^{-9})$, 
\begin{equation}
\label{eqn: epsilon graph constants}
    \alpha_n = \mathcal{O}\left ( \frac{\log(n)}{n}\right )^{\frac{1}{d+4}}, \quad \text{and}\quad\beta_n = \mathcal{O}\left ( \frac{\log(n)}{n}\right )^{\frac{1}{d+4}},
\end{equation}
where the implied constants depend on the geometry of the manifold $\mathcal{M}$. 
\end{theorem}

Additionally, we need the following theorem, which is a consequence of Bernstein's inequality, which will help us prove that $\lim_{n\rightarrow\infty}\widehat{P_nf}(i)=\widehat{f}(i)$ (since $\widehat{f}(i)=\langle f,\phi_i \rangle_{L^2(\mathcal{M})}$ and $\widehat{\mathbf{x}}(i)=\langle\mathbf{x},\phi_i^n\rangle_2$). 

\begin{proposition}\label{lem: bernstein}
Let $f$ and $g$ be continuous functions defined on $\mathcal{M}$. Then for sufficiently large $n$, with probability at least $1 - \frac{2}{n^{9}}$,  we have
    \[|\langle P_n f, P_n g\rangle_2 - \langle f, g \rangle_{L^2(\mathcal{M})}| \leq 6 \sqrt{\frac{\log(n)}{n}}\|f\|_{L^4(\mathcal{M})}\|g\|_{L^4(\mathcal{M})},\]
    where $\|f\|_{L^4(\mathcal{M})}=\left(\int_{\mathcal{M}}|f|^4d\mu\right)^{1/4}$.
\end{proposition}

\begin{proof}

Define random variables $\{X_i\}_{i=1}^n$ by $X_i \coloneqq f(x_i)g(x_i)$ and note that by definition we have 
\[\langle P_n f, P_n g \rangle_2 = \frac{1}{n} \sum_{i = 1}^n f(x_i)g(x_i) = \frac{1}{n}\sum_{i=1}^n X_i.\]
 Since the $x_i$ are sampled i.i.d.~with density $\rho$, we have 
\[\mathbb{E}[X_i] = \int_\mathcal{M}f(x)g(x)\rho(x)dx=\langle f, g \rangle_{L^2(\mathcal{M})}.\]
Therefore, letting $\sigma^2 \coloneqq \mathbb{E}[X_i^2] - \mathbb{E}[X_i]^2$ and $M \coloneqq \|fg - \mathbb{E}[X_i]\|_{L^\infty(\mathcal{M})}$, we see that by Bernstein's inequality, we have
\begin{align*}
    \mathbb{P}(|\langle P_n f, P_n g \rangle_2 - \langle f, g \rangle_{L^2(\mathcal{M})}| > \eta ) &= \mathbb{P}\left(\left | \frac{1}{n}\sum_{i=1}^n X_i - \frac{1}{n} \sum_{i=1}^n \mathbb{E}[X_i]\right | > \eta \right) \\
    &= \mathbb{P}\left(\left | \sum_{i=1}^n X_i - \sum_{i=1}^n \mathbb{E}[X_i]\right | > n\eta \right)\\
    &\leq 2 \exp \left (- \frac{\frac{1}{2}n\eta^2}{\sigma^2 + \frac{1}{3}M\eta}\right ).
\end{align*}
Setting $\eta = 6 \sqrt{\dfrac{\sigma^2 \log(n)}{n}}$, we see that for $n$ large enough so that $1 + \dfrac{M\eta}{3\sigma^2} < 2$, we have
\begin{align*}
    \mathbb{P}(|\langle P_n f, P_n g \rangle_2 - \langle f, g \rangle_{L^2(\mathcal{M})}| > \eta ) &\leq 2 \exp \left (- \frac{18\sigma^2\log(n)}{\sigma^2 + \frac{1}{3}M\eta}\right ) \\
    &= 2 \exp \left (- \frac{18\log(n)}{1 + \frac{M\eta}{3\sigma^2}}\right ) \\
    &< 2 \exp \left ( -9\log(n)\right ) \\
    &= \frac{2}{n^9}.
\end{align*}
We note that $\sigma^2\leq \mathbb{E}[X_i^2]=\langle f^2,g^2\rangle_{L^2(\mathcal{M})}$.
Therefore, by the Cauchy-Schwarz inequality, with probability at least $1 - \frac{2}{n^9}$, we have
\begin{align*}
    |\langle P_n f, P_n g \rangle_2 - \langle f, g \rangle_{L^2(\mathcal{M})}| &\leq 6  \sqrt{\frac{\sigma^2\log(n)}{n}} \\
    &\leq 6 \sqrt{\frac{\log(n)}{n}} \sqrt{\langle f^2,g^2\rangle_{L^2(\mathcal{M})}} \\
    &\leq 6 \sqrt{\frac{\log(n)}{n}} \|f\|_{L^4(\mathcal{M})} \|g\|_{L^4(\mathcal{M})}.
\end{align*}

\end{proof}

\section{The proof of Theorem \ref{thm: Filter Error short}}

Let $\alpha_{n}=\max_{1\leq i\leq \kappa}|\lambda_i-\lambda_i^n|$ and $\beta_n=\max_{1\leq i\leq \kappa}\|\phi_i^n-P_n\phi_i\|_2$.  Recall from Theorem \ref{thm: recall Calder results}, that with probability $1-\mathcal{O}(n^{-9})$ we have
\begin{equation}
\label{eqn: epsilon graph constants in proof}
    \alpha_n = \mathcal{O}\left ( \frac{\log(n)}{n}\right )^{\frac{1}{d+4}}, \quad \text{and}\quad\beta_n = \mathcal{O}\left ( \frac{\log(n)}{n}\right )^{\frac{1}{d+4}}.
\end{equation}
Let 
$\gamma_n=6\sqrt{\frac{\log(n)}{n}}$ be the constant appearing in Proposition \ref{lem: bernstein}.
We now compute 
\begin{align}
&\|w(\mathbf{L}_n)P_nf-P_nw(\mathcal{L})f\|_2\nonumber\\
= &\left\| \sum_{i=1}^\kappa w(\lambda_i^n)\langle P_nf,\phi_i^n\rangle_2\phi_i^n - \sum_{i=1}^\kappa w(\lambda_i)\langle f,\phi_i\rangle_\mathcal{M}P_n\phi_i\right\|_2\nonumber\\
\leq&\left\| \sum_{i=1}^\kappa (w(\lambda_i^n)-w(\lambda_i))\langle P_nf,\phi_i^n\rangle_2\phi_i^n\right\|_2+\left\|\sum_{i=1}^\kappa w(\lambda_i)(\langle P_nf,\phi_i^n\rangle_2\phi_i^n- \langle f,\phi_i\rangle_\mathcal{M}P_n\phi_i)\right\|_2.\label{eqn: different eigenvectorsv2}
\end{align}

To bound the first term from \eqref{eqn: different eigenvectorsv2}. We use the assumption that $w$ is normalized Lipschitz, the triangle inequality, and the Cauchy-Schwarz inequality to see that if $n$ is large enough so that $\gamma_n\leq 1$, we have 
\begin{align}
    \left\| \sum_{i=1}^\kappa (w(\lambda_i^n) - w(\lambda_i))\langle P_nf,\phi_i^n\rangle_2\phi_i^n\right\|_2
    \leq& \max_{1\leq i \leq \kappa} |w(\lambda_i^n)- w(\lambda_i)| \sum_{i=1}^\kappa \|P_n f\|_2 \|\phi_i^n\|^2_2\nonumber\\
    \leq& \alpha_n \sum_{i=1}^\kappa \|P_n f\|_2 \|\phi_i^n\|^2_2\nonumber\\
    \leq& \kappa\alpha_n  \|P_n f\|_2\nonumber\\ 
    \leq& \kappa  \left(\alpha_n\|f\|_{L^2(\mathcal{M})}+\gamma_n\|f\|_{L^4(\mathcal{M})}\right),\nonumber
\end{align}
where we use the fact that $\|\phi_i^n\|_2^2=1$ and that 
\begin{equation}\label{eqn: pnboundv2}\|P_nf\|_2\leq \left(\|f\|_{L^2(\mathcal{M})}^2 + \gamma_n^2\|f\|_{L^4(\mathcal{M})}^2\right)^{1/2}\leq \|f\|_{L^2(\mathcal{M})}+\gamma_n\|f\|_{L^4(\mathcal{M})}.\end{equation}

Now, turning our attention to the second term from \eqref{eqn: different eigenvectorsv2}, we have 
\begin{align}
&\left\|\sum_{i=1}^\kappa w(\lambda_i)(\langle P_nf,\phi_i^n\rangle_2\phi_i^n- \langle f,\phi_i\rangle_{L^2(\mathcal{M})}P_n\phi_i)\right\|_2\nonumber\\
    \leq& \left\|\sum_{i=1}^\kappa w(\lambda_i)\langle P_nf,\phi_i^n\rangle_2(\phi_i^n-P_n\phi_i)\right\|_2
+\left\|\sum_{i=1}^\kappa w(\lambda_i)(\langle P_nf,\phi_i^n\rangle_2- \langle f,\phi_i\rangle_{L^2(\mathcal{M})}P_n\phi_i\right\|_2\label{eqn: use Hoeffding this timev2}.
\end{align}

By definition, we have $\|\phi_i^n-P_n\phi_i\|_2\leq \beta_n$.
Therefore, since $\|w\|_{L^\infty([0,\infty))]}\leq 1$, we see that if $n$ is large enough so $\beta_n\leq 1$ then
\begin{align}
\left\|\sum_{i=1}^\kappa w(\lambda_i)\langle P_nf,\phi_i^n\rangle_2(\phi_i^n-P_n\phi_i)\right\|_2
    &\leq \kappa \max_{1\leq i\leq \kappa} |\langle P_nf,\phi_i^n\rangle_2|\|\phi_i^n-P_n\phi_i\|_2\nonumber\\
    &\leq \kappa\beta_n\|P_nf\|_2\nonumber\\
    &\leq \kappa\left (\beta_n\|f\|_{L^2(\mathcal{M})}+\gamma_n\|f\|_{L^4(\mathcal{M})}\right ),\label{eqn: evec diffv2}
\end{align}
where the final inequality follows from \eqref{eqn: pnboundv2}.
Meanwhile, the second term from \eqref{eqn: use Hoeffding this timev2} can be bounded by 
\begin{align*}
&\left\|\sum_{i=1}^\kappa w(\lambda_i)(\langle P_nf,\phi_i^n\rangle_2- \langle f,\phi_i\rangle_\mathcal{M})P_n\phi_i\right\|_2\\
\leq&\sum_{i=1}^\kappa |w(\lambda_i)| |\langle P_nf,\phi_i^n\rangle_2- \langle f,\phi_i\rangle_\mathcal{M}|\|P_n\phi_i\|_2\\   \leq&\sum_{i=1}^\kappa |\langle P_nf,\phi_i^n\rangle_2- \langle f,\phi_i\rangle_\mathcal{M}|\|P_n\phi_i\|_2\\
\leq&\sum_{i=1}^\kappa |\langle P_nf,\phi_i^n\rangle_2-\langle P_nf,P_n\phi_i\rangle_2|\|P_n\phi_i\|_2+\sum_{i = 1}^\kappa |\langle P_nf,P_n\phi_i\rangle_2- \langle f,\phi_i\rangle_\mathcal{M}|\|P_n\phi_i\|_2.
\end{align*}
 By the Cauchy-Schwarz inequality, Proposition \ref{lem: bernstein},
 \eqref{eqn: epsilon graph constants in proof},  \eqref{eqn: pnboundv2}, and the assumption that $n$ is large enough so that $\beta_n\leq 1$, we have
\begin{align*}
    |\langle P_nf,\phi_i^n\rangle_2-\langle P_nf,P_n\phi_i\rangle_2| &\leq \| P_nf\|_2 \|\phi_i^n-P_n\phi_i\|_2 \\
    &\leq \beta_n\left(\|f\|_{L^2(\mathcal{M})}+\gamma_n\|f\|_{L^4(\mathcal{M})}\right) \\
    &\leq\left(\beta_n\|f\|_{L^2(\mathcal{M})}+\gamma_n\|f\|_{L^4(\mathcal{M})}\right).
\end{align*}
And also by Proposition \ref{lem: bernstein} and \eqref{eqn: pnboundv2} we have 
\begin{align*}
    |\langle P_nf,P_n\phi_i\rangle_2- \langle f,\phi_i\rangle_2| \leq \gamma_n^2\|f\|_{L^4(\mathcal{M})}\|\phi_i\|_{L^4(\mathcal{M})},\quad \text{and} \quad \|P_n\phi_i\|_2\leq 1+\gamma_n\|\phi_i\|_{L^4(\mathcal{M})}.\end{align*}
Since $\kappa$ is fixed and  $\lim_{n\rightarrow\infty}\gamma_n=0,$ for sufficiently large $n$ we have $\gamma_n\|\phi_i\|_{L^4(\mathcal{M})}\leq 1$ for all $i\leq \kappa.$ This implies 
\begin{align*}
    |\langle P_nf,P_n\phi_i\rangle_2- \langle f,\phi_i\rangle_2| \leq \gamma_n\|f\|_{L^4(\mathcal{M})},\quad\text{and}\quad
\|P_n\phi_i\|_2\leq 1+\gamma_n\|\phi_i\|_{L^4(\mathcal{M})}\leq 2.
\end{align*}

Therefore, the second term from \eqref{eqn: use Hoeffding this timev2} can be bounded by

\begin{align}
    &\left\|\sum_{i=1}^\kappa w(\lambda_i)(\langle P_nf,\phi_i^n\rangle_2 -\langle f,\phi_i\rangle_\mathcal{M})P_n\phi_i\right\|_2 \nonumber\\
    \leq&\sum_{i=1}^\kappa |\langle P_nf,\phi_i^n\rangle_2-\langle P_nf,P_n\phi_i\rangle_2|\|P_n\phi_i\|_2 \nonumber +\sum_{i=1}^\kappa|\langle P_nf,P_n\phi_i\rangle_2- \langle f,\phi_i\rangle_2|\|P_n\phi_i\|_2\nonumber\\
\leq&\sum_{i=1}^\kappa\left(\beta_n\|f\|_{L^2(\mathcal{M})}+\gamma_n\|f\|_{L^4(\mathcal{M})}\right)\|P_n\phi_i\|_2 +\sum_{i=1}^\kappa  \gamma_n\|f\|_{L^4(\mathcal{M})}\|P_n\phi_i\|_2\nonumber\\
\leq&4\kappa\left( \beta_n\|f\|_{L^2(\mathcal{M})}+\gamma_n\|f\|_{L^4(\mathcal{M})}\right).\label{eqn: inner pdt diffv2}
\end{align}

Therefore, combining Equations \eqref{eqn: different eigenvectorsv2} through \eqref{eqn: inner pdt diffv2} yields 
\begin{align*}
&\|w(\mathbf{L}_n)P_nf-P_nw(\mathcal{L})f\|_2\nonumber\\
\leq &\left\| \sum_{i=1}^\kappa (w(\lambda_i^n) - w(\lambda_i))\langle P_nf,\phi_i^n\rangle_2\phi_i^n\right\|_2+\left\|\sum_{i=1}^\kappa w(\lambda_i)(\langle P_nf,\phi_i^n\rangle_2\phi_i^n-\langle f,\phi_i\rangle_\mathcal{M}P_n\phi_i)\right\|_2\\
\leq& \kappa  (\alpha_n\|f\|_{L^2(\mathcal{M})}+\gamma_n\|f\|_{L^4(\mathcal{M})})+ 5\kappa \left(\beta_n\|f\|_{L^2(\mathcal{M})}+\gamma_n\|f\|_{L^4(\mathcal{M})}\right)\\
\leq &
6\kappa\left((\alpha_n+\beta_n)\|f\|_{L^2(\mathcal{M})}+\gamma_n\|f\|_{L^4(\mathcal{M})}\right).
\end{align*}
Plugging in the values of $\alpha_n, \beta_n,$ and $\gamma_n$ completes the proof.
\section{The Proof of Theorem \ref{thm: bound given filter bound short}}

\begin{lemma}\label{lem: x not f}
Under the assumptions of Theorem \ref{thm: bound given filter bound short}, we have  
\begin{align*}
\|w(\mathbf{L}_n)\mathbf{x}-P_nw(\mathcal{L})f\|_2
\leq& \|\mathbf{x}-P_nf\|_2 + 
6\kappa\left((\alpha_n+\beta_n)\|f\|_{L^2(\mathcal{M})}+\gamma_n\|f\|_{L^4(\mathcal{M})}\right),\end{align*} where
$\alpha_n, \beta_n$ and $\gamma_n$ are as in the proof of Theorem \ref{thm: Filter Error short}
for all continuous functions $f$ and all $\mathbf{x}\in\mathbb{R}^n.$ (We do not assume $\mathbf{x}=P_nf$ here.)
\end{lemma}
\begin{proof}
We first observe that since $\|w\|_{L^\infty([0,\infty))} \leq 1,$ we have \begin{align}
%\sum_{q=1}^{F_{\ell-1}}
\|w(\mathbf{L}_n)\mathbf{x}-w(\mathbf{L}_n)P_nf\|_2 &=\|w(\mathbf{L}_n)(\mathbf{x}-P_nf)\|_2 \nonumber\\&=\|\sum_{i=1}^n w(\lambda_i^n) \langle \mathbf{x}-P_nf, \phi_i^n\rangle_2\phi_i^n\|_2\\
&=\left(\sum_{i=1}^n |w(\lambda_i^n)|^2|\langle \mathbf{x}-P_nf, \phi_i^n\rangle_2|^2\right)^{1/2}\nonumber\\
&\leq \left(\sum_{i=1}^n |\langle \mathbf{x}-P_nf, \phi_i^n\rangle_2|^2\right)^{1/2}\nonumber\\
%F_{\ell-1}
&=\|\mathbf{x}-P_nf\|_2.\label{eqn: easy recursion termv2}
\end{align}
Therefore, by the triangle inequality, together with Theorem \ref{thm: Filter Error short}, we we have \begin{align*}
\|w(\mathbf{L}_n)\mathbf{x}-P_nw(\mathcal{L})f\|_2
\leq& %\sum_{q=1}^{F_{\ell-1}}
\|w(\mathbf{L}_n)\mathbf{x}-w(\mathbf{L}_n)P_nf\|_2 +%\sum_{q=1}^{F_{\ell-1}}
\|w(\mathbf{L}_n)P_nf-P_nw(\mathcal{L})f\|_2\\
\leq& \|\mathbf{x}-P_nf\|_2 + 
6\kappa\left((\alpha_n+\beta_n)\|f\|_{L^2(\mathcal{M})}+\gamma_n\|f\|_{L^4(\mathcal{M})}\right)\end{align*}
as desired.
\end{proof}

\begin{proof}[The proof of Theorem \ref{thm: bound given filter bound short}]

Let $$\epsilon_{n}=\mathcal{O}\left(\left( \frac{\log(n)}{n}\right)^{\frac{1}{d+4}}\right)\max_k\|f_k\|_{L^2(\mathcal{M})}+\mathcal{O}\left(\left(\frac{\log(n)}{n}\right)^{1/4}\right)\max_k\|f_k\|_{L^4(\mathcal{M})},
$$
so that by Lemma \ref{lem: x not f}, together with the definitions of $\alpha_n, \beta_n,$ and $\gamma_n$ we have 
$$
\|w(\mathbf{L}_n)\mathbf{x}-P_nw(\mathcal{L})f_k\|_2\leq \epsilon_n.
$$
Let $A_1^{(\ell)}=\max_{j,k}(|\sum_{i=1}^{C_{\ell}} |\theta_{i,k}^{(\ell,j)}|),$ $A_2^{(\ell)}=\max_{j,k}(\sum_{i=1}^{J_{\ell}} |\alpha_{j, i}^{(\ell,k)}|)$, where  $\theta_{i,k}^{(\ell,j)}$ and $\alpha_{j, i}^{(\ell,k)}$ are the weights used in the combination step (step (ii)) and the cross-filter step (step (iii) of the MFCN. (See Appendix \ref{app: MFCN}.)
The following lemma bounds the error induced in the non-filtering steps of the discrete MFCN implementation.

\begin{lemma}\label{lem: other step error}
The errors induced by the non-filtering steps of our network may be bounded by 
\begin{align}
   \| \mathbf{y}_{j,k}^{(\ell)}-P_ng_{j,k}^{(\ell)}\|_2
&\leq \max_{1\leq i\leq C_{\ell}}\|\tilde{\mathbf{x}}^{(\ell)}_{j,k}-P_n\tilde{f}^{(\ell)}_{j,k}\|_2\sum_{i=1}^{C_{\ell}} |\theta_{i,k}^{(\ell,j)}|,\label{eqn: theta error}
\\
\| \tilde{\mathbf{y}}_{j,k}^{(\ell)}-P_n\tilde{g}_{j,k}^{(\ell)}\|_2
&\leq \max_{1\leq i\leq J_{\ell}}\|\mathbf{y}^{(\ell)}_{j,k}-P_n g^{(\ell)}_{j,k}\|_2\sum_{i=1}^{J_{\ell}} |\alpha_{j,i}^{(\ell,k)}|.\label{eqn: alpha error}\\
\|\mathbf{z}^{(\ell)}_{j,k}-P_nh^{(\ell)}_{j,k}\|_2&\leq\|\tilde{\mathbf{y}}^{(\ell)}_{j,k}-P_n\tilde{g}^{(\ell)}_{j,k}\|_2.\label{eqn: sigma error}    
\end{align}
(where the notation is as in Appendix \ref{app: MFCN}).
\end{lemma}
\begin{proof}
To verify \eqref{eqn: theta error}, we observe that 
    \begin{align*}
   \| \mathbf{y}_{j,k}^{(\ell)}-P_ng_{j,k}^{(\ell)}\|_2
&=\left\|\sum_{i=1}^{C_{\ell}}\tilde{\mathbf{x}}^{(\ell)}_{j,k}\theta_{i,k}^{(\ell,j)}-P_n\tilde{f}^{(\ell)}_{j,k}\theta_{i,k}^{(\ell,j)}\right\|_2\\
&\leq \sum_{i=1}^{C_{\ell}}|\theta_{i,k}^{(\ell,j)}|\|\tilde{\mathbf{x}}^{(\ell)}_{j,k}-P_n\tilde{f}^{(\ell)}_{j,k}\|_2\nonumber\\
&\leq  \max_{1\leq i\leq C_{\ell}}\|\tilde{\mathbf{x}}^{(\ell)}_{j,k}-P_n\tilde{f}^{(\ell)}_{j,k}\|_2\sum_{i=1}^{C_{\ell}} |\theta_{i,k}^{(\ell,j)}|.%\\
   \end{align*}
The proof of \eqref{eqn: alpha error} is identical to the proof of \eqref{eqn: theta error}. For \eqref{eqn: sigma error}, we see that 
since $\sigma$ is non-expansive we have 
\begin{align*}
\|\mathbf{z}^{(\ell)}_{j,k}-P_nh^{(\ell)}_{j,k}\|^2_2
&=\sum_{i=1}^n|
(\mathbf{z}^{(\ell)}_{j,k})(i)-(P_nh^{(\ell)}_{j,k})(i)|^2\\
&=\sum_{i=1}^n|
(\mathbf{z}^{(\ell)}_{j,k})(i)-h^{(\ell)}_{j,k}(x_i)|^2\\
&=\sum_{i=1}^n|
\sigma((\tilde{\mathbf{y}}^{(\ell)}_{j,k})(i))-\sigma(\tilde{g}^{(\ell)}_{j,k}(x_i))|^2\\
&\leq\sum_{i=1}^n|
(\tilde{\mathbf{y}}^{(\ell)}_{j,k})(i)-\tilde{g}^{(\ell)}_{j,k}(x_i)|^2\\
&=\|\tilde{\mathbf{y}}^{(\ell)}_{j,k}-P_n\tilde{g}^{(\ell)}_{j,k}\|^2_2.
\end{align*}

\end{proof}
Now, returning to the proof of Theorem \ref{thm: bound given filter bound short}, it follows from the definition of the reshaping operator $$\max_k\|\mathbf{x}_k^{(\ell+1)}-P_nf_k^{(\ell+1)}\|_2
= \max_{j,k}\|\mathbf{z}^{(\ell)}_{p,r}-P_nh^{(\ell)}_{p,r}\|_2.$$
(Again, see Appendix \ref{app: MFCN} for notation.)
Therefore, by Lemma \ref{lem: other step error} we have 
\begin{align*}
\max_k\|\mathbf{x}_k^{(\ell+1)}-P_nf_k^{(\ell+1)}\|_2
=& \max_{j,k}\|\mathbf{z}^{(\ell)}_{p,r}-P_nh^{(\ell)}_{p,r}\|_2.\\
\leq& \|P_n\tilde{g}^{(\ell)}_{j,k}-\tilde{\mathbf{y}}^{(\ell)}_{j,k}\|_2.\\
\leq& A^{(\ell)}_2\max_{j,k}\|P_n g^{(\ell)}_{j,k}-\mathbf{y}^{(\ell)}_{j,k}\|_2\\
\leq& A^{(\ell)}_2A^{(\ell)}_1\max_{j,k}\|P_n \tilde{f}^{(\ell)}_{j,k}-\tilde{\mathbf{x}}^{(\ell)}_{j,k}\|_2\\
\leq& A^{(\ell)}_2A^{(\ell)}_1(\max_k\|\mathbf{x}_k^{(\ell)}-P_nf_k^{(\ell)}\|_2
+\epsilon_{n}).
\end{align*}

Since $\|\mathbf{x}_0^{(\ell)}-P_nf^{(0)}_k\|_2=0$ for all $k$,
we may use induction to conclude that 
\begin{equation*}\|\mathbf{x}^{(\ell)}_{k}-P_nf^{(\ell)}_k\|_2\leq \sum_{i=0}^{\ell-1} \prod_{j=i}^{\ell-1} A_{1}^{(j)} A_{2}^{(j)} \epsilon_{n}.
\end{equation*}
In particular, if we assume that $A_{1}^{(j)} A_{2}^{(j)}\leq A$, we have 
$$
\|\mathbf{x}^{(\ell)}_{k}-P_nf^{(\ell)}_k\|_2\leq \ell A^\ell \epsilon_{n}.
$$
and if $A=1,$ we have 
$$
\|\mathbf{x}^{(\ell)}_{k}-P_nf^{(\ell)}_k\|_2\leq \ell \epsilon_{n}.
$$
\end{proof}

\begin{remark}
Results similar to Theorems \ref{thm: Filter Error short} and \ref{thm: bound given filter bound short} can also be derived in the setting where the filters $w$ are bandlimited, i.e., $w(\lambda_i)=0\:\forall i>\kappa$. They can also be obtained when $G_n$ is constructed as a $k$-NN graph. Additionally, similar results can be derived if the filters have Lipschitz constants greater than one (where the bound will depend on the largest Lipschitz constant).
\end{remark}

\section{Experiments and Extension to $k$-NN Graphs}\label{app: numerics}

In Figure \ref{fig:eps_spectral_filter_converg_plot}, we numerically demonstrate the convergence of a spectral filter, complementing our theoretical analysis in Theorem \ref{thm: Filter Error short}. We focus on the two-dimensional unit sphere, embedded in $\mathbb{R}^3$, with uniform sampling, since in this setting there are known formulas for the eigenvalues and eigenfunctions (i.e., spherical harmonics) which allows us to compare our numerical implementation to ground truth.  We consider a simple function $f$ chosen to be the sum of two eigenfunctions i.e., $f=Y^0_1+Y^0_2$, where $Y^i_j$ is the $i$-th harmonic of degree $j$. We sample $n$ points uniformly from the sphere evaluated $f$ at these points, and built the graph $G_n$ as described in Section \ref{sec: MFCN discrete implementation}. We then applied the spectral filter $w(\lambda)=e^{-\lambda}$. 

When computing the eigendecomposition, we used only the first $64$ eigenpairs as $w(\lambda)$ takes negligible values at the higher eigenvalues. %  and zero-ed out the other eigenvalues. For each sample, we also constructed a simple spectral filter in the manner described in Section \ref{sec: MFCN discrete implementation}, building sparse $\epsilon$-graphs and computing graph Laplacians and their spectral decompositions (to 64 eigenpairs), and choosing $w(\lambda^n_i) = e^{-\lambda^n_i}$. 
We repeated this process $10$ times for each value of $n$ and calculated the discretization errors of the spectral filter as $\|w(\mathbf{L}_n) P_n f - P_n w(\mathcal{L}) f\|_2$ for increasingly large values of $n$.
Additionally, we also tracked the convergence of the first two distinct, non-zero eigenvalues (corresponding to the first eight eigenvalues counting their multiplicity), %since on the sphere, $\lambda_i$ has multiplicity $2i + 1$ for $i \geq 0$), 
as shown in Figure \ref{fig:eps_eigenvals_converg_plot}. We see that the numerical eigenvalues converge to the true values of $\frac{\ell(\ell+1)}{8\pi}$, for $\ell=1,2$. (Please note that with uniform sampling on the two-dimensional unit sphere, $\mathcal{L}f=-\frac{1}{2\rho}\text{div}(\rho^2\nabla f)$ reduces to $-\frac{1}{8\pi}\Delta$, where $-\Delta=-\text{div}\circ\nabla$ is the negative Laplace-Beltrami operator.)

\subsection{Extension to $k$-NN graphs}
In the main text, for the sake of concreteness, we have focused on $\epsilon$-graph constructions. However, our results may be straightforwardly extended to (symmetric) $k$-NN graphs, $G_n=(V_n,E_n)$, in which $\{x_i,x_j\}\in E_n$ if $x_i$ is one of the $k$-th closest points to $x_j$ (with respect to Euclidean distance in $\mathbb{R}^D$) or if $x_j$ is one of the $k$ closest points to $x_i$.

\begin{figure*}
    \begin{center}
    \includegraphics[width=1.0\textwidth]{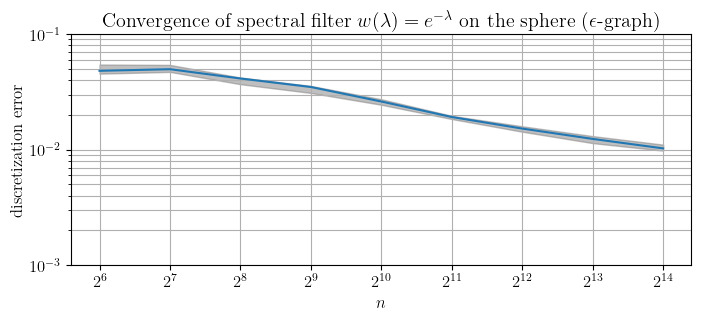}
    \caption{Discretization error for spectral filter $w(\lambda)=e^{-\lambda}$ applied to the sum of two spherical harmonics, for an $\epsilon$-graph construction. The median error of 10 runs is shown in blue, against a gray band of the 25th- to 75th-percentile error range.}
    \label{fig:eps_spectral_filter_converg_plot}
    \end{center}
    % \vspace{-20pt}
    %\vspace{1pt}
\end{figure*}
\begin{figure*}
    \begin{center}
    \includegraphics[width=1.0\textwidth]{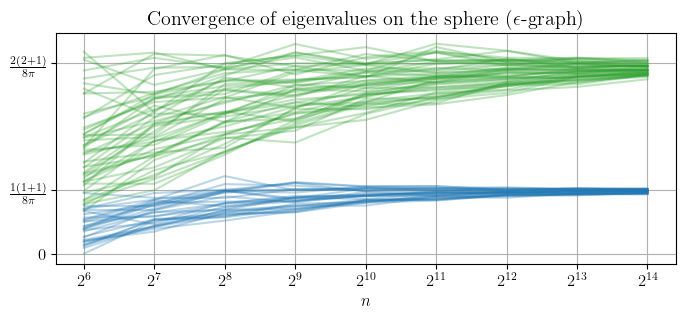}
    \caption{Convergence of first eight, non-zero eigenvalues on the sphere, for an $\epsilon$-graph construction, all 10 runs combined. The blue lines are the first three eigenvalues that converge to the same limit (since the first non-zero eigenvalue of the spherical Laplacian has multiplicity three). Similarly, the green lines are the next five eigenvalues. }
    \label{fig:eps_eigenvals_converg_plot}
    \end{center}
    % \vspace{-20pt}
    %\vspace{1pt}
\end{figure*}

\begin{figure*}
    \begin{center}
    \includegraphics[width=1.0\textwidth]{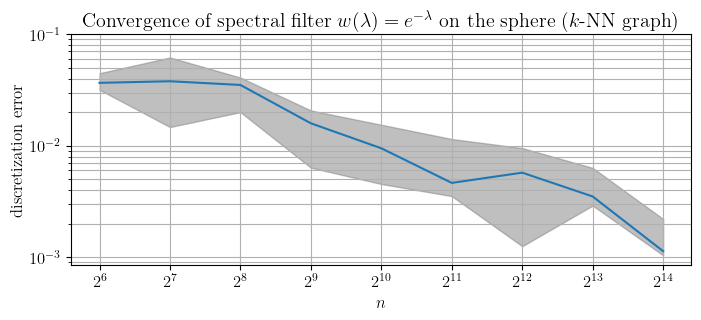}
    \caption{Discretization error for spectral filter $w(\lambda)=e^{-\lambda}$ applied to the sum of two spherical harmonics, for a $k$-NN graph construction. The median error of 10 runs is shown in blue, against a gray band of the 25th- to 75th-percentile error range.}
    \label{fig:knn_spectral_filter_converg_plot}
    \end{center}
    % \vspace{-20pt}
    %\vspace{1pt}
\end{figure*}
\begin{figure*}
    \begin{center}
    \includegraphics[width=1.0\textwidth]{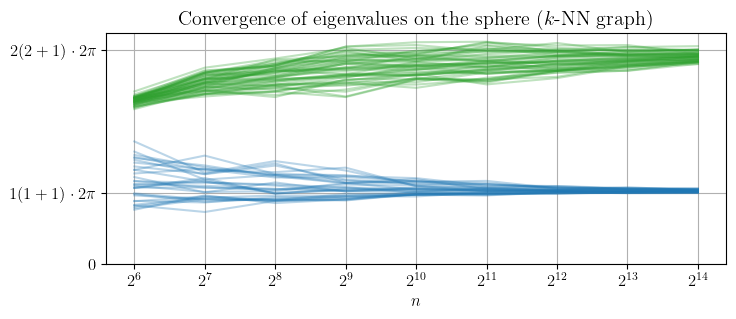}
    \caption{Convergence of first eight, non-zero eigenvalues on the sphere, for a $k$-NN graph construction, all 10 runs combined. The blue lines are the first three eigenvalues that converge to the same limit (since the first non-zero eigenvalue of the spherical Laplacian has multiplicity three). Similarly, the green lines are the next five eigenvalues.}
    \label{fig:knn_eigenvals_converg_plot}
    \end{center}
\end{figure*}

In this setting, one defines the graph Laplacian by $
\mathbf{L}_n=\frac{d+2}{v_d n}\left(\frac{nv_d}{k}\right)^{1+2/d}\left(\mathbf{D}_{n}-\mathbf{A}_{n}\right),
$
where $v_d$ is the volume of the $d$-dimensional Euclidean unit ball (and $\mathbf{D}_{n}$ and $\mathbf{A}_{n}$ are the degree and adjacency matrices associated to $G_n$), and the limiting manifold Laplacian is given by $\mathcal{L}f=-\frac{1}{2\rho}\text{div}(\rho^{1-2/d}\nabla f)$, which reduces to $-2\pi\Delta$ when sampling uniformly on the sphere. If one sets $k \sim \log(n)^{\frac{d}{d+4}} n^{\frac{4}{d+4}}$, then one may readily derive results analogous to 
 Theorems \ref{thm: Filter Error short}
 and \ref{thm: bound given filter bound short}. The only difference in the proof is that we must invoke Theorems 2.5 and 2.9 of \citet{Calder2019}, rather than Theorems 2.4 and 2.7 (which we restate as Theorem \ref{thm: recall Calder results}). We illustrate this numerically in Figures \ref{fig:knn_spectral_filter_converg_plot}  and \ref{fig:knn_eigenvals_converg_plot}, which are the analogs of Figures \ref{fig:eps_spectral_filter_converg_plot} and \ref{fig:eps_eigenvals_converg_plot}, but with the graph constructed as a $k$-NN graph rather than an $\epsilon$ graph.

\end{document}